\documentclass{article}
\usepackage[utf8]{inputenc}

\usepackage{amsmath}
\usepackage{amssymb}
\usepackage{amsthm}
\usepackage{xcolor}
\usepackage{subcaption}
\usepackage{graphicx}

\newtheorem{theorem}{Theorem}
\newtheorem{proposition}{Proposition}
\newtheorem{remark}{Remark}

\usepackage[ruled, lined, linesnumbered, commentsnumbered, longend]{algorithm2e}
\newcommand{\Tr}{\mathrm{Tr}}
\newcommand{\sys}{Eqs. \eqref{eq:fullargmin}-\eqref{eq:norm}\ }
\newcommand{\ppow}[1]{^{(#1)}}
\newcommand{\argmin}{\text{arg\,min}}

\title{Continuous Semi-Supervised Nonnegative Matrix Factorization}
\author{Michael R. Lindstrom \thanks{The University of Texas Rio Grande Valley, School of Mathematical and Statistical Sciences, mike.lindstrom@utrgv.edu} 
\and Xiaofu Ding \thanks{University of California Los Angeles Department of Mathematics}  \and Feng Liu \footnotemark[2] \and Anand Somayajula \footnotemark[2] \and Deanna Needell \footnotemark[2]}

\date{\today}

\begin{document}

\maketitle

\begin{abstract}
    Nonnegative matrix factorization can be used to automatically detect topics within a corpus in an unsupervised fashion. The technique amounts to an approximation of a nonnegative matrix as the product of two nonnegative matrices of lower rank. In this paper, we show this factorization can be combined with regression on a continuous response variable. In practice, the method performs better than regression done after topics are identified and retrains interpretability.
\end{abstract}

{\bf Keywords:} topic modelling, regression, nonnegative matrix factorization, optimization

\section{Introduction}

Nonnegative Matrix Factorization (NMF) is a highly versatile data science technique with far-reaching applications. It can identify thematic elements, i.e., groups of words that appear frequently together in a corpus, which together convey a common message. More generally, it can be used to decompose an image into identifiable patterns \cite{lee1999learning} and as a general-purpose dimensionality reduction or preprocessing method before applying other machine learning methods as has been done in studying various diseases \cite{lao2021regression,lai2013survival}. Like Singular Value Decomposition (SVD) \cite{stewart1993early}, NMF provides a low rank factorization. In NMF, a nonnegative matrix $X \in \mathbb{R}^{n \times m}_{\geq 0}$ representing a corpus (or other nonnegative dataset) is factored into a low rank approximation $X \approx WH$ where the inner dimension, $r$, between $W$ and $H$ is such that $r \ll m$ and $r \ll n$; however, unlike SVD, there is an additional constraint that both $W$ and $H$ are nonnegative, i.e., $W \in \mathbb{R}^{n \times r}_{\geq 0}$ and $H \in \mathbb{R}^{r \times m}_{\geq 0}.$ This non-negativity enforces that the data in $X$ is represented by a non-negative combination of the dictionary atoms in the factorization, which lends itself to human interpretability. For example, in the foundational work \cite{lee1999learning}, Lee and Seung show that NMF when applied to facial images decompose the images into recognizable parts such as noses, eyes, and mouths.

When applied to a document-term matrix $X$ where row $i$ of $X$ represents document $i$ and column $j$ represents the frequency of word $j$, the classical NMF method amounts to 
\begin{equation}
(W, H) = \argmin_{W \in \mathbb{R}^{n \times r}_{\geq 0}, H \in \mathbb{R}^{r \times m}_{\geq 0}} ||X - W H||_F^2 \label{eq:basic}
\end{equation}
where the $F$-subscript denotes the Frobenius norm with $||A||_F^2 = \Tr(A^T A) = \sum_{i=1}^n \sum_{j=1}^m |A_{ij}|^2$. Other variations on the penalty norm exist including the Kulblack-Liebler Divergence \cite{joyce2011kullback}. After computing $W$ and $H$, we interpret row $j$ of $H$ as the $j$th topic, its components being the weight of each word in that topic, and row $i$ of $W$ as the topic-encoding of document $j$, i.e.,
$$X_{i,:} \approx \sum_{j=1}^r W_{i,j} H_{j,:}.$$ Throughout this manuscript we make use of ``colon notation" where ``:" means the full range of indices for a row/column, ``a:b" indicates a consecutive range of indices from a to b, etc.

Prior authors have combined NMF with a linear regression procedure to maximize the predictive power of a classifier \cite{austin2018fully,zhu2018joint,asil21wsdm,nncpdvhn20}. This is accomplished through a penalty function that combines NMF with another objective function --- a (semi) supervised approach. Semi-supervised NMF can also be applied to guide NMF to identify topics with desired keywords \cite{gssnmf22}.

In this paper, we combine this NMF with a linear regression model to predict the value of a continuous response variable. We consider datasets that pair written commentary with a real-valued observation. As our motivating example, we consider \texttt{Rate My Professor} reviews \cite{rmp_data} that include all student comments for a professor along with the mean rating in $[1,5]$. Due to the averaging, the rating is effectively a continuous variable.

The rest of our paper is organized as follows: in section \ref{sec:formulate}, we provide the formulation of our method, its algorithmic implementation, and its theoretical properties; proofs of the properties are given in section \ref{sec:proofs}; in section \ref{sec:concept}, we provide a proof of concept through synthetic data; in section \ref{sec:rmp}, we test our method on the \texttt{Rate My Professor} dataset; and finally we conclude our work in section \ref{sec:conclusion}.

\section{Model}

\label{sec:formulate}

We provide the framework for our proposed Continuous Semi-Supervised Nonnegative Matrix Factorization method (CSSNMF).

\subsection{Formulation}

We consider having a document-term corpus $X \in \mathbb{R}_{\geq 0}^{n \times m}$ for $n$ documents with their associated word frequencies in the $m$ columns. Each document has a corresponding value in $\mathbb{R}$ so that we can associate with $X$ the vector $Y \in \mathbb{R}^n$. We choose $r \in \mathbb{N}$ and $\lambda \geq 0$ as hyper-parameters where $r$ denotes the number of topics and $\lambda$ is a regression weight. Given $W \in \mathbb{R}_{\geq 0}^{n \times r}$, $H \in \mathbb{R}_{\geq 0}^{r \times m}$, and $\theta \in \mathbb{R}^{r+1}$, we define the penalty function that combines topic modelling with a linear regression based on the topic representation.

\begin{align}
    F\ppow{\lambda}(W,H,\theta;X,Y) &= N(W,H;X) + \lambda R(W,\theta;Y), \text{ where } \label{eq:F} \\
    N(W,H;X) &:= ||X - WH||_F^2 \label{eq:N} \\
    R(W,\theta;Y) &:= ||\tilde W \theta - Y||^2 \label{eq:R}
\end{align}
and where $\bar W \in \mathbb{R}^{n \times (r+1)}$ is given by
\begin{equation}
    \bar W := 
    \begin{pmatrix} 1 & W_{1,1} & \hdots & W_{1,m} \\
    1 & W_{2,1} & \hdots & W_{2,m} \\
    \vdots & \vdots & \vdots & \vdots \\
    1 & W_{n,1} & \hdots & W_{n,m}
    \end{pmatrix}. \label{eq:Wtilde}
\end{equation}

The matrix $\bar W$ with its column of 1's allows for an intercept: given a topic representation $w \in \mathbb{R}^r$, we predict a value $\hat y = \theta_1 + \theta_2 w_1 + ... + \theta_{r+1} w_r.$

{\it When} $\lambda>0$, we seek
\begin{equation}
(W\ppow{\lambda}, H\ppow{\lambda}, \theta\ppow{\lambda}) = \argmin_{W,H,\theta} F\ppow{\lambda}(W,H,\theta;X,Y). \label{eq:fullargmin}
\end{equation}
And {\it when} $\lambda=0$, we define
\begin{align}
(W\ppow{0}, H\ppow{0}) &= \argmin_{W,H} N(W,H;X) \label{eq:WHargmin} \\
\theta\ppow{0} &= \argmin_\theta R(W\ppow{0},\theta;Y). \label{eq:thetaargmin}
\end{align}

We also impose a normalization constraint, that 
\begin{equation}
\forall i, \quad \sum_{j=1}^m H_{ij} = 1 \label{eq:norm}
\end{equation}
so that the topics have unit length in $\ell_1.$ If $X \approx WH$ is normalized so its rows sum to $1$ then it is also the case that $\forall i, \quad \sum_{j=1}^r W_{ij} \approx 1$ by noting that $X_{ij} = \sum_{k=1}^r W_{ik} H_{kj}$ and summing over $j$.

When $\lambda=0$, $\theta$ has no effect upon $F\ppow{\lambda}$ and we first perform regular NMF over $W$ and $H$ and, as a final step, we choose $\theta$ to minimize the regression error. In other words, if $\lambda=0$, we do NMF first and then find the best $\theta$ given the already determined weights for each document. It seems intuitive, however, that the regression could be improved if $\theta$ and $W$ both were being influenced by the regression to $Y$, which is what our method aims to do when $\lambda > 0$. From a practical perspective, if $\lambda \uparrow \infty$ then the regression error becomes dominant and we may expect the topics as found through $H$ to be less meaningful. In section \ref{ssec:theory}, we state some theoretical properties of our method as it is being trained.

Once $H\ppow{\lambda}$ and $\theta\ppow{\lambda}$ are known, we can make predictions for the response variable corresponding to a document. This amounts to finding the best nonnegative topic encoding $w \in \mathbb{R}^r$ for the document and using that encoding in the linear model --- see section \ref{ssec:algs}.

\begin{remark}[Uniqueness]
Using our established notation, we remark that if $X^* = WH$ and $Y^* = \theta_1 + W \theta_{2:(r+1)}$ then $X^* = \tilde W \tilde H$ and $Y^* = \theta_1 + \tilde W \tilde \theta$ where $\tilde W = SW$, $\tilde H = S^{-1} H$, and $\tilde \theta = S^{-1} \theta_{2:(r+1)}$ for any invertible $S \in \mathbb{R}^{r \times r}_{\geq 0}$ with $S^{-1} \in \mathbb{R}^{r \times r}_{\geq 0}$. Thus, uniqueness of an optima, if it exists, can only be unique up to matrix multiplications. \label{rmk:unique}
\end{remark}

\subsection{Theoretical Results}

\label{ssec:theory}

We present two important behaviours of CSSNMF with regards to increasing $\lambda$ and its effect upon predicting the response variable. The proofs are contained in section \ref{sec:proofs}.

\begin{proposition}[Regression Error with Nonzero $\lambda$]
For $\lambda \geq 0$, let $W^{(\lambda)}, H^{(\lambda)}, \theta^{(\lambda)}$ be a unique (as per Remark \ref{rmk:unique}) global minimum to \sys. Then \\ $R(W^{(\lambda)},\theta^{(\lambda)}) \leq R(W^{(0)},\theta^{(0})$. \label{prop:nonzero}
\end{proposition}

\begin{theorem}[Weakly Decreasing Regression Error]
Let $0 \leq \lambda_1 < \lambda_2$ be given where $W\ppow{\lambda_i}, H\ppow{\lambda_i}, \theta\ppow{\lambda_i}$ are the unique (as per Remark \ref{rmk:unique}) global minimizers of \sys for $i=1,2$. Then $R(W\ppow{\lambda_2},\theta\ppow{\lambda_2};Y) \leq R(W\ppow{\lambda_1},\theta\ppow{\lambda_1};Y).$ \label{thm:mono}
\end{theorem}

\begin{remark}
Proposition \ref{prop:nonzero} and Theorem \ref{thm:mono} are based on obtaining a global minimum. In practice, we may only find a local minimum. 
\end{remark}

Assuming we have the optimal solutions, Proposition \ref{prop:nonzero} tells us that the regression error for $\lambda>0$ is no worse than the regression error with $\lambda=0$ and could in fact be better. Thus, the intuition that selecting topics while paying attention to the regression error is practical. Then Theorem \ref{thm:mono} says that the regression error is weakly monotonically decreasing as $\lambda$ increases. In practical application, we find the error in fact strictly monotonically decreases.

\subsection{Algorithm}

\label{ssec:algs}

Our minimization approach is iterative and based on the alternating nonnegative least squares \cite{kim2008nonnegative} approach. Due to the coupling of NMF and Regression errors, other approaches such as multiplicative or additive updates \cite{lee2000algorithms} are less natural. Each iteration consists of: (1) holding $H$ and $\theta$ fixed while optimizing each row of $W$ separately (non-negative least squares); (2) holding $W$ and $\theta$ fixed while optimizing each column of $H$ separately (non-negative least squares); and finally (3) holding $W$ and $H$ fixed while optimizing over $\theta$. The error is nondecreasing between iterations and from one optimization to the next. We now derive and justify this approach (Algorithm \ref{alg:overall}) in increasing complexity of cases.

\paragraph{$W$ and $H$ fixed.} If $W$ and $H$ are given and only $\theta$ can vary then Eq. \eqref{eq:F} is minimized when $||\bar W \theta - Y||^2$ is minimized. This happens when the error, $\bar W \theta - Y$, is orthogonal to the column span of $W$ or that \begin{equation}
\theta = (\bar W^T \bar W)^{-1} \bar W^T Y. \label{eq:theta_opt}
\end{equation}
See Algorithm \ref{alg:newt}. When $\bar W$ does not have full rank, we interpret $(\bar W^T \bar W)^{-1}$ as a pseudo-inverse.

\paragraph{$W$ and $\theta$ fixed.} If $W$ and $\theta$ are given and only $H$ can change then minimizing Eq. \eqref{eq:F} requires minimizing $N(W,H;X)$. We can expand this error term out in the columns of $H$:

\begin{align*}
N(W,H;X) &= ||X - WH||_F^2 \\
&= \sum_{j=1}^m ||(X - WH)_{:,j}||^2 \\
&= \sum_{j=1}^m ||X_{:,j} - W H_{:,j}||^2.
\end{align*}
Since columnwise the terms of the sum are independent, we can minimize each column $H_{:,j}$ of $H$ separately to minimize the sum, i.e.,
\begin{align}
H_{:,j} = \argmin_{h \in \mathbb{R}_{\geq 0}^{r \times 1}} ||\bar X_{:,j} - Wh||^2, \quad j=1,2,...,m,
\end{align} as given in Algorithm \ref{alg:newH}.

\paragraph{$H$ and $\theta$ fixed.} When $H$ and $\theta$ are fixed then \eqref{eq:F} can be written out as

\begin{align}
F\ppow{\lambda} &= ||X - WH||_F^2 + \lambda ||\bar W \theta - Y||^2 \nonumber \\ 
&= \sum_{i=1}^n ||(X - WH)_{i,:}||^2 + \lambda \sum_{i=1}^n (\bar W \theta - Y)_i^2 \nonumber \\
&= \sum_{i=1}^n ||X_{i,:} - W_{i,:}H||^2 + \sum_{i=1}^n \left( \sqrt{\lambda} (\theta_0 e + W_{i,:} \bar \theta - Y) \right)_i^2 \label{eq:rowwise}
\end{align}
where $e = (1,1,...,1)^T \in \mathbb{R}^{n}$ and $\bar \theta = (\theta_2, ..., \theta_{r+1})^T \in \mathbb{R}^{r}$. Defining matrices
\begin{align}
\bar X &= \begin{bmatrix} X & | & \sqrt{\lambda}(\theta_1 e - Y) \end{bmatrix} \label{eq:xhat} \\
\bar H &= \begin{bmatrix} H & | & \sqrt{\lambda} \bar \theta \end{bmatrix} \label{eq:hhat}
\end{align}
we can rewrite \eqref{eq:rowwise} as 
\begin{align*}
F\ppow{\lambda} &= \sum_{i=1}^n ||\bar X_{i,:} - W_{i,:} \bar H||^2,
\end{align*}
which can be minimized through
\begin{align}
W_{i,:} = \argmin_{w \in \mathbb{R}_{\geq 0}^{1 \times r}} ||\bar X_{i,:} - w \bar H||^2, \quad i=1,2,...,n.
\end{align} This is precisely Algorithm \ref{alg:newW}.

\begin{algorithm}
    \SetKwInOut{KwIn}{Input}
    \SetKwInOut{KwOut}{Output}
    \KwIn{A matrix $X \in \mathbb{R}_{\geq 0}^{n \times m}$,\\ a vector $Y \in \mathbb{R}^n$, \\ a positive integer $r \in \mathbb{N}$, \\
    a scalar $\lambda \geq 0$, \\
    a relative error tolerance $\tau > 0$, and \\
    a maximum number of iterations $maxIter$.
    }
    \KwOut{Minimizers of \sys: nonnegative matrix $W \in \mathbb{R}_{\geq 0}^{n \times r}$, \\
    nonnegative matrix $H \in \mathbb{R}_{\geq 0}^{r \times m}$, and \\
    vector $\theta \in \mathbb{R}^{r+1}.$}
    $relErr = \infty, err = \infty$ \\
    Elementwise, $W \sim Unif([0,||X||_\infty))$, $H \sim Unif([0,||X||_\infty))$, $\theta \sim Unif([0,||X||_\infty)).$ \\
    $iter=0$ \\
    \While{$relErr > \tau$ {\bf and} $iter<maxIter$}{
    $W \gets newW$ as per Algorithm \ref{alg:newW} \\
    $H \gets newH$ as per Algorithm \ref{alg:newH} \\
    $\theta \gets new{}\theta$ as per Algorithm \ref{alg:newt} \\
    Normalize $W$, $H$, and $\theta$ as per Algorithm \ref{alg:norm} \\
    $errTemp = F\ppow{\lambda}(W,H,\theta;X,Y)$ \\
    \If{$err<\infty$}{
$relErr \gets |err-errTemp|/err$
    } 
    $err \gets errTemp$\\
    $iter \gets iter+1$
}
    \KwRet{$W,H,\theta$}
    \caption{Overall CSSNMF algorithm.} \label{alg:overall}
\end{algorithm}

\begin{algorithm}
    \SetKwInOut{KwIn}{Input}
    \SetKwInOut{KwOut}{Output}
    \KwIn{A matrix $X \in \mathbb{R}_{\geq 0}^{n \times m}$,\\ a vector $Y \in \mathbb{R}^n$, \\
    a matrix $W \in \mathbb{R}_{\geq 0}^{n \times r}$, \\
    a matrix $H \in \mathbb{R}_{\geq 0}^{r \times m}$, \\ and
    a scalar $\lambda \geq 0.$    
    }
    \KwOut{A new value for $W$.}
    $\bar \theta = (\theta_2, ..., \theta_{r+1})^T$ \\
    $\bar X = \begin{bmatrix} X & | & \sqrt{\lambda}(\theta_1 - Y) \end{bmatrix}$ \\
    $\bar W = \begin{bmatrix} H & | & \sqrt{\lambda} \bar \theta \end{bmatrix}$ \\
    \For{$i \gets 1 ... n$}{
    $W_{i,:} \gets \argmin_{w \in \mathbb{R}_{\geq 0}^{1 \times r}} ||\bar X_{i,:} - w \bar H||^2$
    }
    \KwRet{$W$}
    \caption{Updating $W$.} \label{alg:newW}
\end{algorithm}

\begin{algorithm}
    \SetKwInOut{KwIn}{Input}
    \SetKwInOut{KwOut}{Output}
    \KwIn{A matrix $X \in \mathbb{R}_{\geq 0}^{n \times m}$, \\
    a matrix $W \in \mathbb{R}_{\geq 0}^{n \times r}$, and \\
    a matrix $H \in \mathbb{R}_{\geq 0}^{r \times m}.$    
    }
    \KwOut{A new value for $H$.}
    \For{$j \gets 1 ... m$}{
    $H_{:,j} \gets \argmin_{h \in \mathbb{R}_{\geq 0}^{r \times m}} ||\bar X_{:,j} - Wh||^2$
    }
    \KwRet{$H$}
    \caption{Updating $H$.} \label{alg:newH}
\end{algorithm}

\begin{algorithm}
    \SetKwInOut{KwIn}{Input}
    \SetKwInOut{KwOut}{Output}
    \KwIn{A vector $Y \in \mathbb{R}^n$, and \\
    a matrix $W \in \mathbb{R}_{\geq 0}^{n \times r}$    
    }
    \KwOut{A new value for $\theta$.}
    $e = (1,1,...,1)^T \in \mathbb{R}^{n \times 1}$ \\
    $\bar W = \begin{bmatrix} e & | & W \end{bmatrix}$ \\
    \KwRet{$(\bar W^T \bar W)^{-1} \bar W^T Y$}
    \caption{Updating $\theta$.} \label{alg:newt}
\end{algorithm}

\begin{algorithm}
    \SetKwInOut{KwIn}{Input}
    \SetKwInOut{KwOut}{Output}
    \KwIn{A matrix $W \in \mathbb{R}_{\geq 0}^{n \times r}$, \\
    a matrix $H \in \mathbb{R}_{\geq 0}^{r \times m}$, and \\
    a vector $\theta \in \mathbb{R}^{r+1}.$
    }
    \KwOut{New values for $W$, $H$, and $\theta$.}
    $S \in \mathbb{R}_{\geq 0}^{r}$ a vector of row sums of $H$. \\
    $S \gets \mathrm{diag}(S)$ \\
    $W \gets W S$. \\
    $H \gets S^{-1} H$. \\
    $\theta_{2:r+1} \gets S^{-1} \theta_{2:r+1}.$ \\
    \KwRet{$W$, $H$, and $\theta$.}
    \caption{Normalization process.} \label{alg:norm}
\end{algorithm}

\begin{algorithm}
    \SetKwInOut{KwIn}{Input}
    \SetKwInOut{KwOut}{Output}
    \KwIn{A matrix $H \in \mathbb{R}_{\geq 0}^{r \times m}$, \\
    a vector $\theta \in \mathbb{R}^{r+1}$, \\ and 
    a vector $x \in \mathbb{R}^{1 \times m}$.
    }
    \KwOut{Model prediction for response variable, $\hat y$.}
    Compute $w = \argmin_{w \in \mathbb{R}_{\geq 0}^{1 \times r}} ||wH - x||^2$. \\
    Compute $\hat y = \theta_1 + w \theta_{2:r+1}$. \\
    \KwRet{$\hat y$}
    \caption{Prediction process.} \label{alg:predict}
\end{algorithm}

For our optimizations and linear algebra, we used \texttt{Numpy} \cite{harris2020array} and \texttt{SciPy} \cite{2020SciPy-NMeth}. Besides the steps outlined within these algorithms, we employed two additional modifications: (1) we defined $\epsilon = 10^{-10}$ and any entries in the $H$ less than $\epsilon$ were replaced by $\epsilon$ ( otherwise on some occasions, the $W$ update step would fail); and (2), the minimizations at times yielded worse objective errors than already obtained and when this happened, we did not update to the worse value.

As noted with other NMF routines, we might not reach a global minimizer \cite{berry2007algorithms}. In practice the minimization should be run repeatedly with different random initializations to find a more ideal local minimum.

From an application standpoint, we wish to run the model on documents it has not been trained on. Algorithm \ref{alg:predict} stipulates how a prediction takes place. We first find the best nonnegative decomposition of the document, a vector in $\mathbb{R}^{m}$, into the topic basis, projecting to $r-$dimensions. With the representation in topic-coordinates, we then use the linear model.

\section{Proofs}\label{sec:proofs}

Before proceeding to practical applications, we prove Proposition \ref{prop:nonzero} and Theorem \ref{thm:mono}.

\begin{proof}[Proof of Proposition \ref{prop:nonzero}]
If $\lambda>0$ then
\begin{align*}
F\ppow{\lambda}(W\ppow{\lambda},H\ppow{\lambda},\theta\ppow{\lambda};X,Y) &\leq F\ppow{\lambda}(W\ppow{0},H\ppow{0},\theta\ppow{0};X,Y) \implies \\
N(W\ppow{\lambda},H\ppow{\lambda};X) + \lambda R(W\ppow{\lambda},\theta\ppow{\lambda};Y) &\leq N(W\ppow{0},H\ppow{0};X) + \lambda R(W\ppow{0},\theta\ppow{0};Y) \\
&\implies \\
\lambda(R(W\ppow{\lambda},\theta\ppow{\lambda};Y) - R(W\ppow{0},\theta\ppow{0};Y)) &\leq N(W\ppow{0},H\ppow{0};X) - N(W\ppow{\lambda},H\ppow{\lambda};X) \\
&\leq 0.
\end{align*}
The first inequality comes from how $(W\ppow{\lambda},H\ppow{\lambda},\theta\ppow{\theta})$ are defined by Eq. \eqref{eq:fullargmin}. The final inequality comes from how $(W\ppow{0},H\ppow{0})$ are defined as minimizers in Eq. \eqref{eq:WHargmin}. 

Since we first assumed $\lambda>0$, we obtain $R(W\ppow{\lambda},\theta\ppow{\lambda};Y) \leq R(W\ppow{0},\theta\ppow{0};Y).$ Finally if $\lambda=0$ then there is equality with $R(W\ppow{\lambda},\theta\ppow{\lambda};Y) = R(W\ppow{0},\theta\ppow{0};Y).$
\end{proof}

\begin{proof}[Proof of Theorem \ref{thm:mono}]
Note that if $\lambda_1 = 0$ then Theorem \ref{prop:nonzero} already applies so we assume $0<\lambda_1<\lambda_2$. We have that
\begin{align}
&F\ppow{\lambda_1}(W\ppow{\lambda_1},H\ppow{\lambda_1},\theta\ppow{\lambda_1};X,Y) \leq F\ppow{\lambda_1}(W\ppow{\lambda_2},H\ppow{\lambda_2},\theta\ppow{\lambda_2};X,Y) \implies& \nonumber \\ 
&\lambda_1\left( R(W\ppow{\lambda_1},\theta\ppow{\lambda_1};Y) - R(W\ppow{\lambda_2},\theta\ppow{\lambda_2};Y) \right) \leq& \nonumber \\
& N(W\ppow{\lambda_2},H\ppow{\lambda_2};X) - N(W\ppow{\lambda_1},H\ppow{\lambda_1};X).& \label{eq:e12}
\end{align}
We also have
\begin{multline}
\lambda_2\left( R(W\ppow{\lambda_2},\theta\ppow{\lambda_2};Y) - R(W\ppow{\lambda_1},\theta\ppow{\lambda_1};Y) \right) \leq  \\
 N(W\ppow{\lambda_1},H\ppow{\lambda_1};X) - N(W\ppow{\lambda_2},H\ppow{\lambda_2};X). \label{eq:e21}
\end{multline}
Adding Eqs. \eqref{eq:e12} and \eqref{eq:e21} together,
\begin{equation*}
(\lambda_1 - \lambda_2) R(W\ppow{\lambda_1},\theta\ppow{\lambda_1};Y) + (\lambda_2 - \lambda_1) R(W\ppow{\lambda_2},\theta\ppow{\lambda_2};Y) \leq 0 
\end{equation*}
which, upon dividing by $\lambda_2-\lambda_1 > 0$, directly gives $$R(W\ppow{\lambda_2},\theta\ppow{\lambda_2};Y) \leq R(W\ppow{\lambda_1},\theta\ppow{\lambda_1};Y).$$
\end{proof}

\section{Synthetic Datasets}

\label{sec:concept}

In our synthetic data, we generate a matrix $X$ that has nonnegative factors $W$ and $H$, but add noise. We also generate a response vector $Y$ given as the matrix-vector product $\bar W \theta$ with noise. We investigate three items: (1) that the method does in fact work to decrease the objective function; (2) that the regression errors decrease with increasing $\lambda$; and (3) the effects of overfitting.

\subsection{Generating Synthetic Data}

Our synthetic data generation can be summarized as follows:
\begin{enumerate}
\item We fix values of $n=100, m=40$, $M=20$, and $r=4.$ 
\item We then define $\eta_x=\eta_y=4.$ 
\item We pick $X \in \mathbb{R}^{n \times r}$ such that each entry is $\sim Unif([0,M))$. We likewise choose $H \in \mathbb{R}^{r \times m}.$
\item We set $X = WH$.
\item We pick $\theta \in \mathbb{R}^{r+1}$ such that each element is $\sim Unif([-M/2,M/2)).$
\item We set $Y = \bar W \theta.$
\item We perturb $X$ with noise $\mathcal{D}_X$ and $Y$ with noise $\sim \mathcal{D}_Y$.
\item Any negative $X$-entries are set to $0$.
\end{enumerate}

We consider two different forms for $\mathcal{D}_X$ and $\mathcal{D}_Y$:
\begin{itemize}
\item being elementwise $\sim \mathcal{N}(0,\eta_x^2)$ and $\sim \mathcal{N}(0,\eta_y^2)$ or
\item being elementwise $\sim Unif([0, \eta_x))$ and $\sim Unif([0, \eta_y))$.
\end{itemize}

Note that in the synthetic data, the true number of topics is $r=4$. In testing our synthetic data, we run Algorithm \ref{alg:overall} where $\tau=10^{-4}$ and $maxIter=100$. We use $70\%$ of the data for training and $30\%$ for testing.

\subsection{Investigation}

We confirm that the error in the objective function $F\ppow{\lambda}$ decreases with each iteration of Algorithm \ref{alg:overall} in Figure \ref{fig:objective_decrease} --- done with Gaussian noise. 

With the regression error being the mean squared prediction error, from Figure \ref{fig:synthetic_errors}, we see the regression error in the training does tend to decrease with $\lambda$. (There are a few small exceptions, which we believe stem from randomizations leading to an assortment of different local optima.) The overall scale of the testing errors gets smaller as $r$ goes from $1$ to $4$ and then stays steady, or even gets slightly worse as $r$ increases from $4$. Indeed $r=4$ is the ``correct" synthetic value. Given the noise as either Gaussian or uniform, the variances of $\mathcal{N}(0, \eta_y^2)$, $\eta_y^2$, and $Unif([0, \eta_y))$, $\eta_y^2/12$, serve as loose estimates for the best possible testing loss (the loss could very well be higher since noise is added to the matrix $X$ as well). When the training errors are smaller than this estimate it suggests overfitting. Since our model works with a sum of squared errors, it is expected when the errors are not Gaussian, that the model will not perform as well. Indeed, there is some degradation in testing errors in comparing the Gaussian with Uniform noise.

\begin{figure}
\centering
\includegraphics[width=0.5\linewidth]{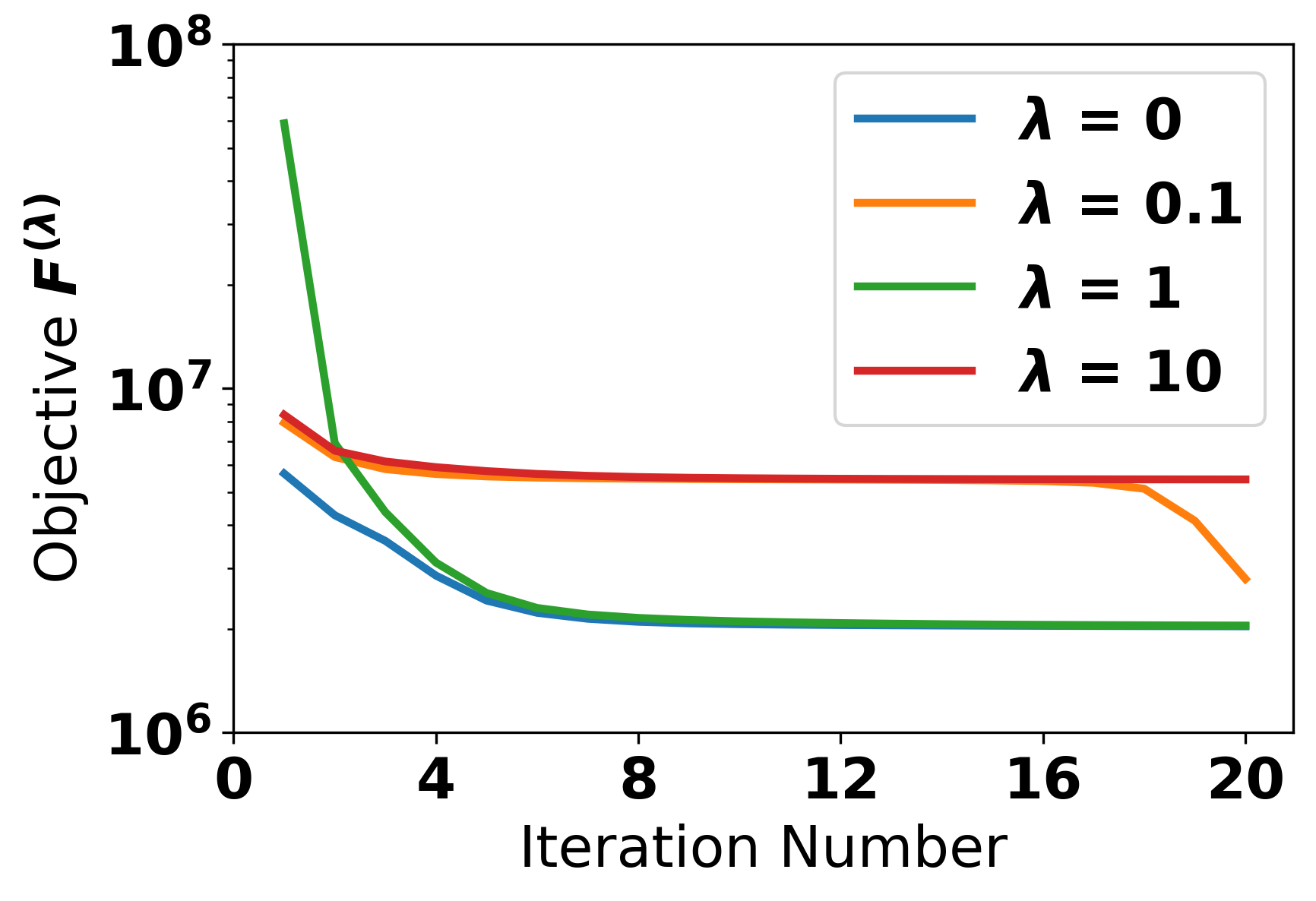}
\caption{Illustration of decreasing objective function at $r=3$ topics for $\lambda \in \{0\} \cup \{10^{i/2} | i \in \mathbb{Z} \cap [-2,2]$ \}.}
\label{fig:objective_decrease}
\end{figure}

\begin{figure}
     \centering
     \begin{subfigure}[b]{0.3\textwidth}
         \centering
         \includegraphics[width=\textwidth]{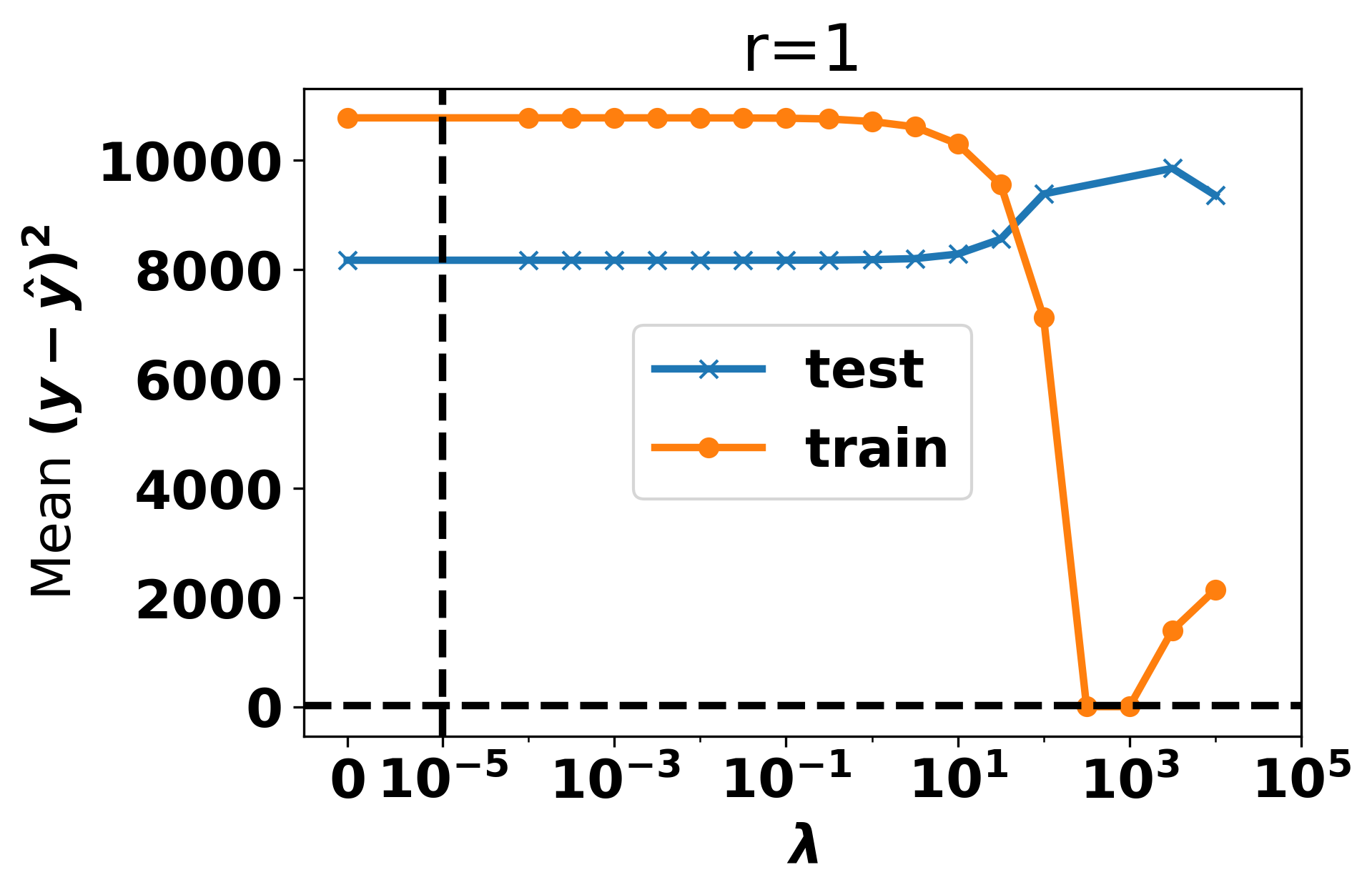}
         \caption{}
         \label{fig:s1}
     \end{subfigure}
     \hfill
     \begin{subfigure}[b]{0.3\textwidth}
         \centering
         \includegraphics[width=\textwidth]{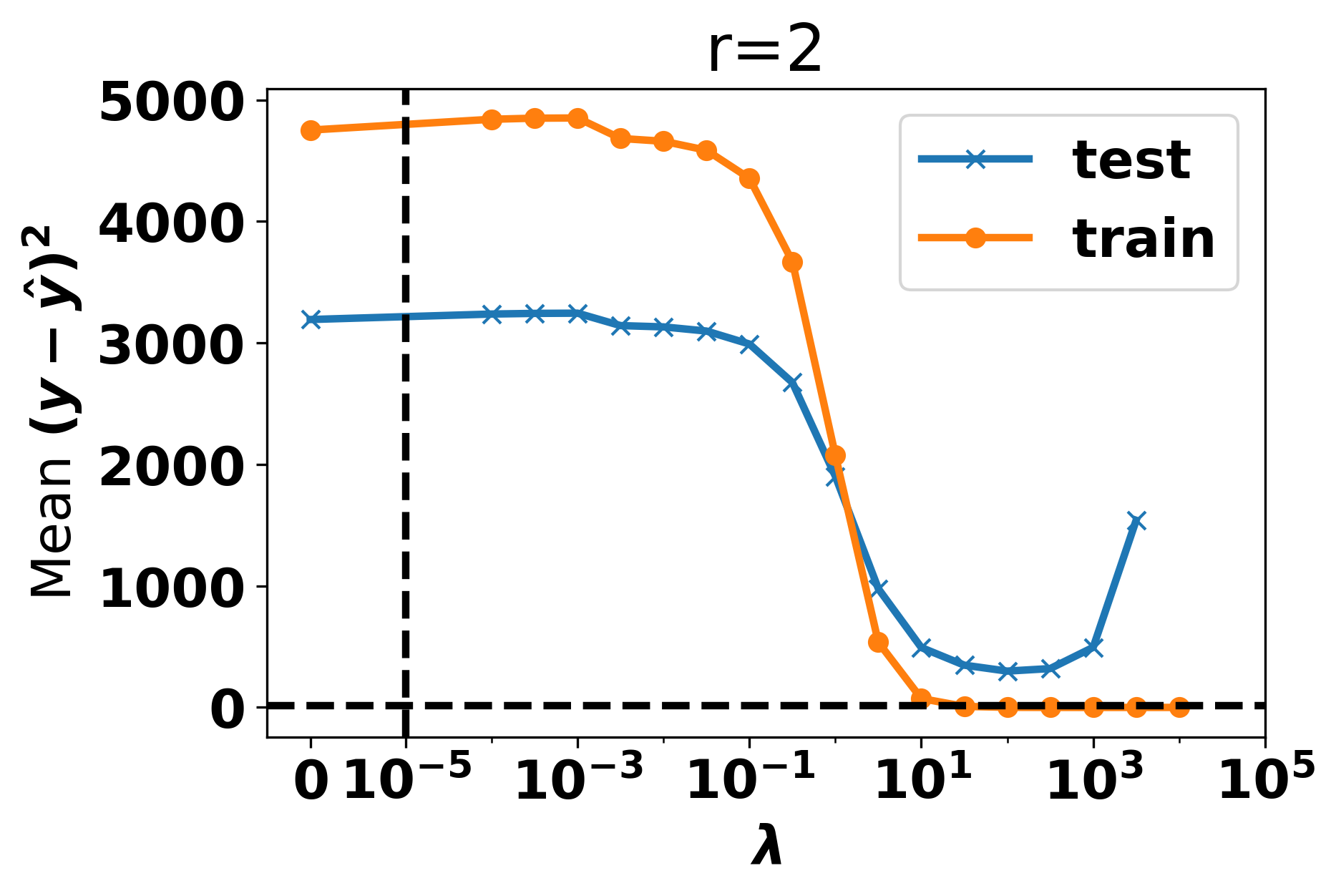}
         \caption{}
         \label{fig:s2}
     \end{subfigure} \hfill
     \begin{subfigure}[b]{0.3\textwidth}
         \centering
         \includegraphics[width=\textwidth]{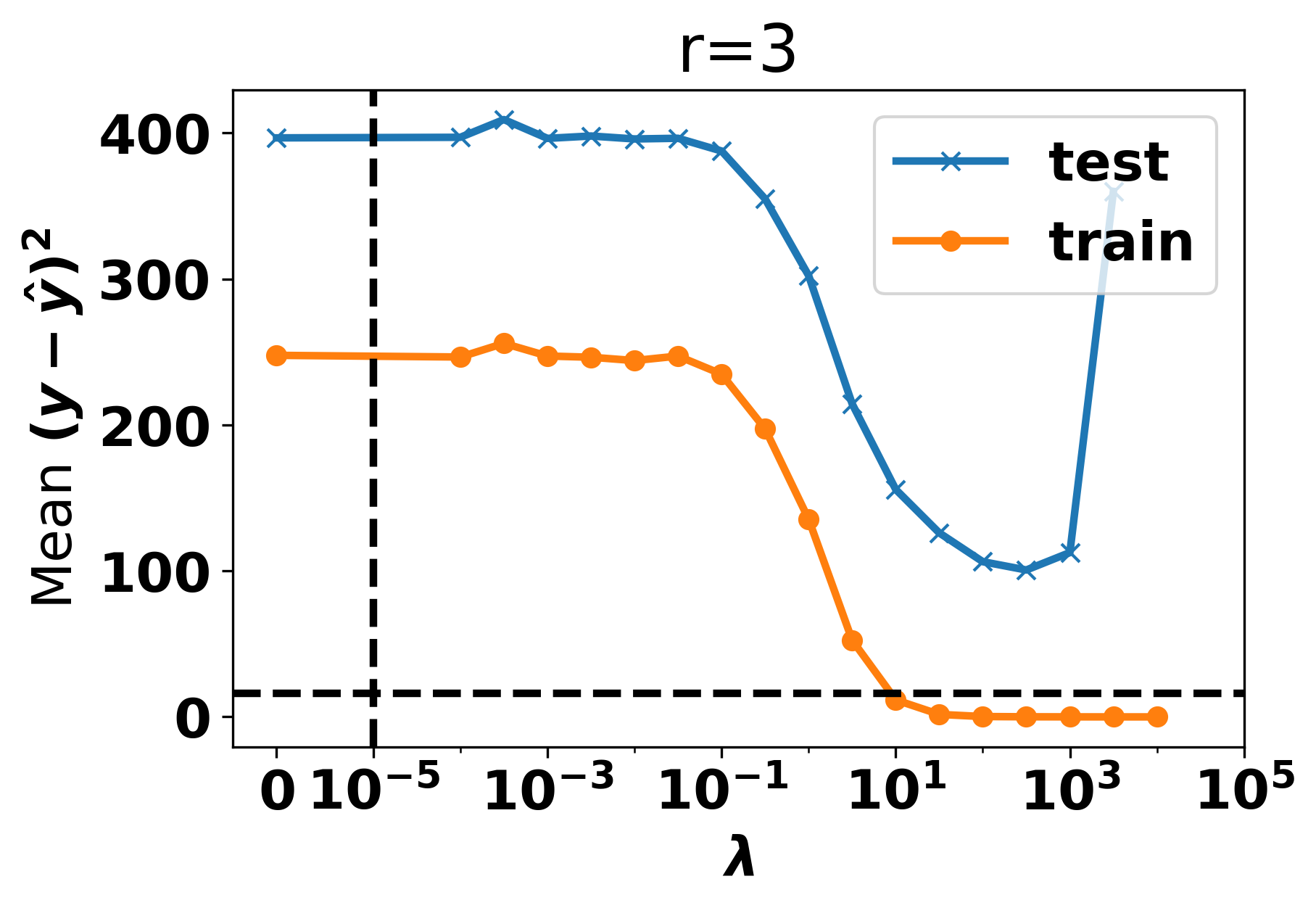}
         \caption{}
         \label{fig:s3}
     \end{subfigure} \\
     \begin{subfigure}[b]{0.3\textwidth}
         \centering
         \includegraphics[width=\textwidth]{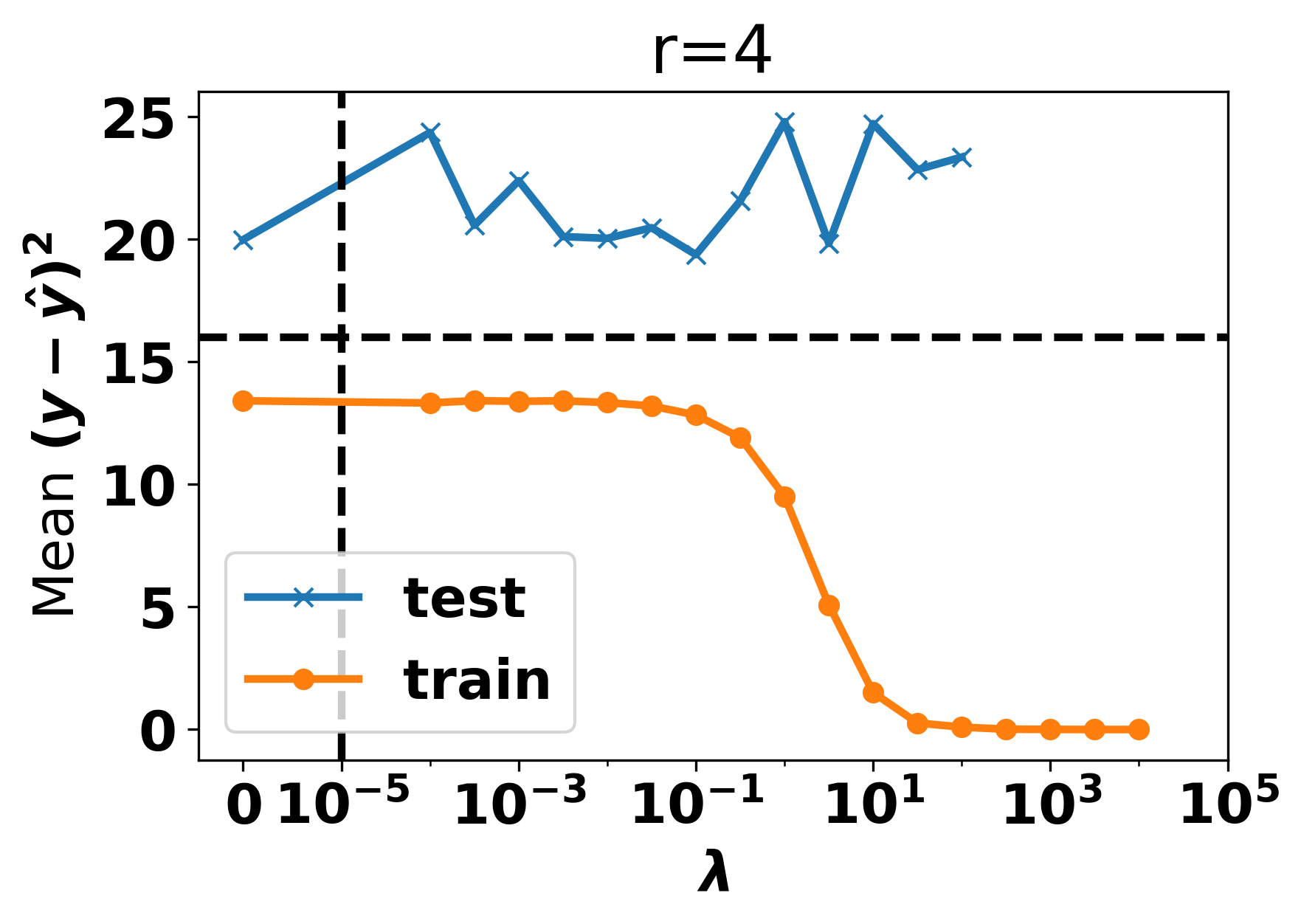}
         \caption{}
         \label{fig:s4}
     \end{subfigure} \hfill
     \begin{subfigure}[b]{0.3\textwidth}
         \centering
         \includegraphics[width=\textwidth]{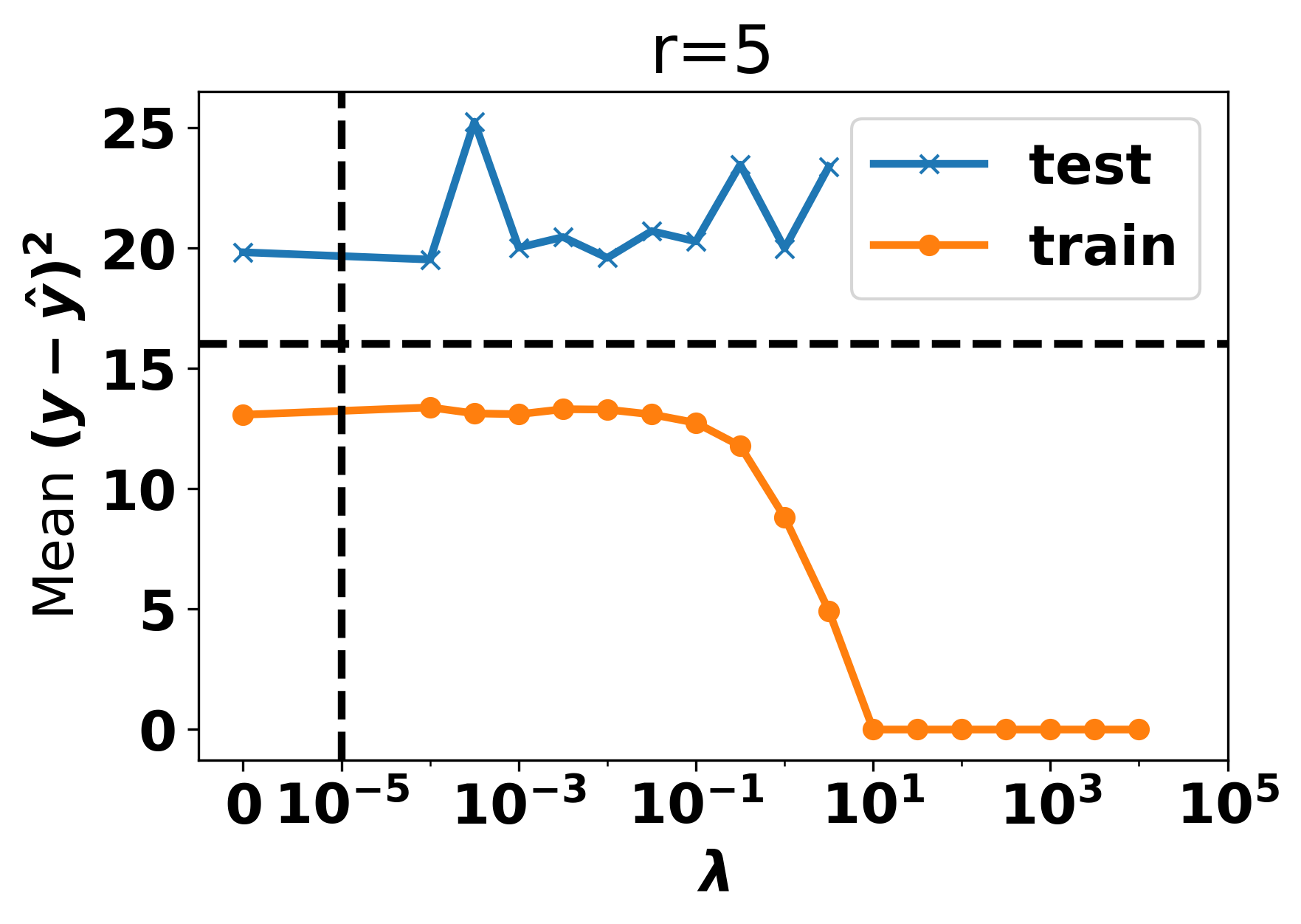}
         \caption{}
         \label{fig:s5}
     \end{subfigure}
     \hfill
     \begin{subfigure}[b]{0.3\textwidth}
         \centering
         \includegraphics[width=\textwidth]{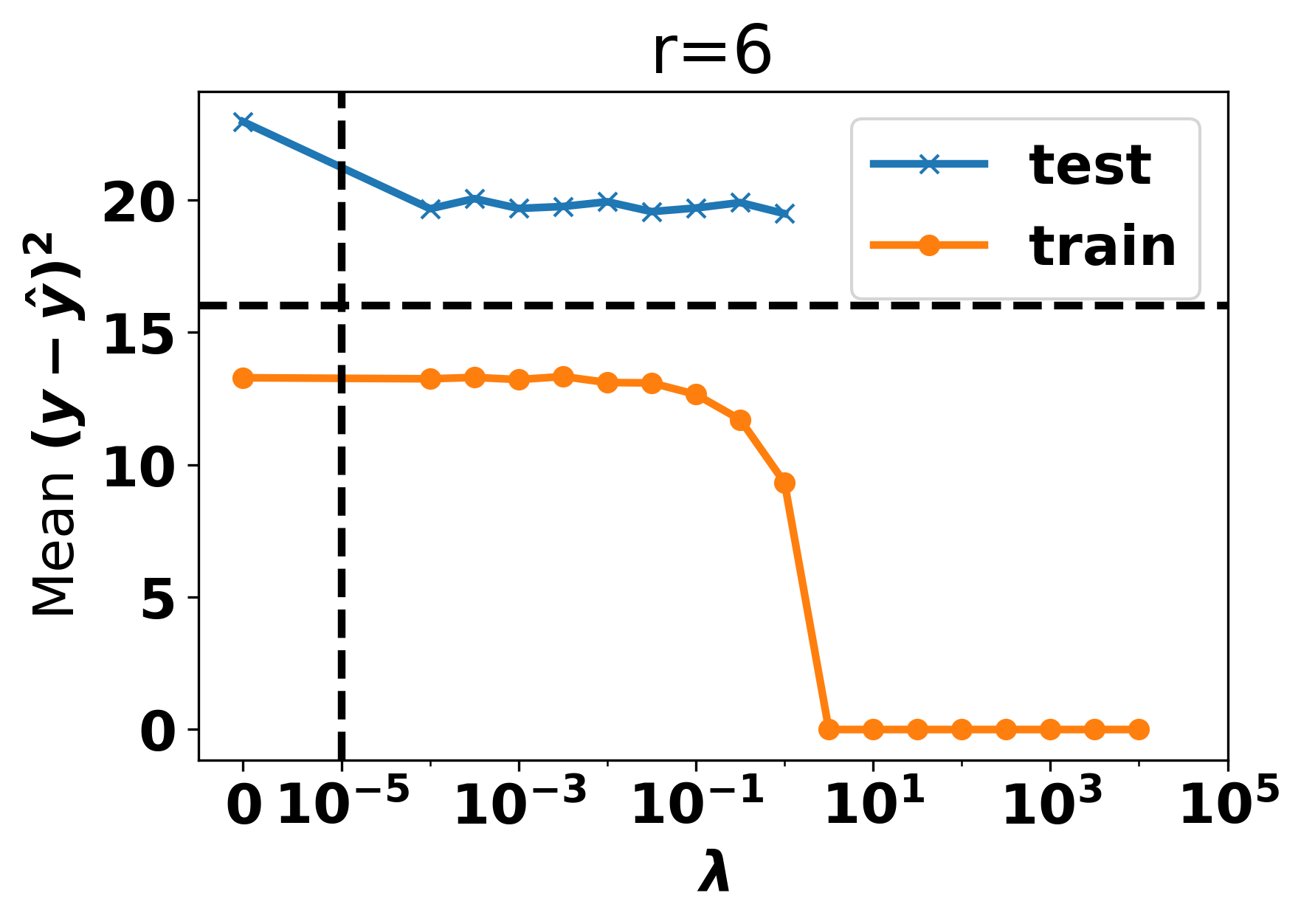}
         \caption{}
         \label{fig:s6}
     \end{subfigure}     \\
     \begin{subfigure}[b]{0.3\textwidth}
         \centering
         \includegraphics[width=\textwidth]{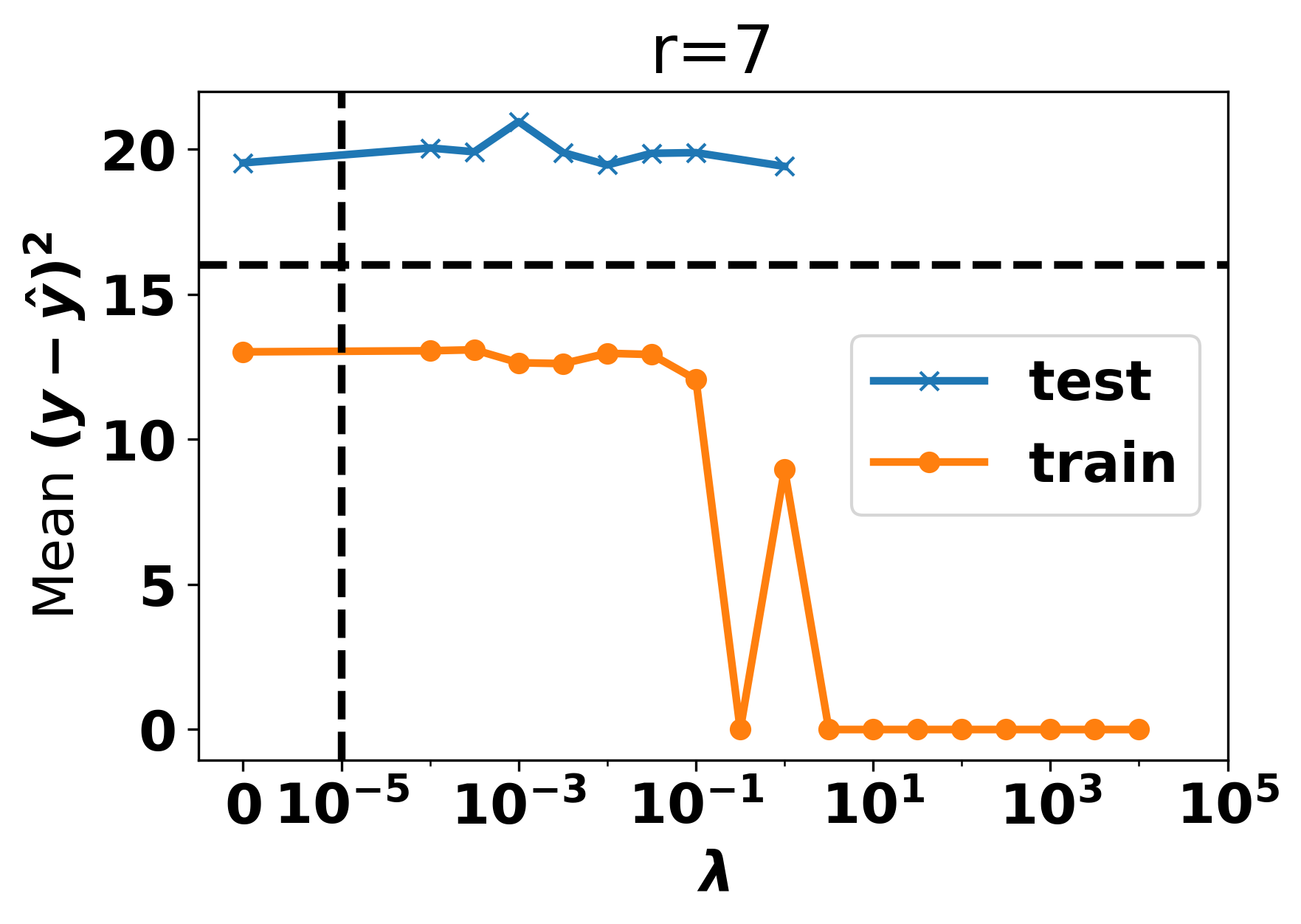}
         \caption{}
         \label{fig:s5}
     \end{subfigure}
     \hfill
     \begin{subfigure}[b]{0.3\textwidth}
         \centering
         \includegraphics[width=\textwidth]{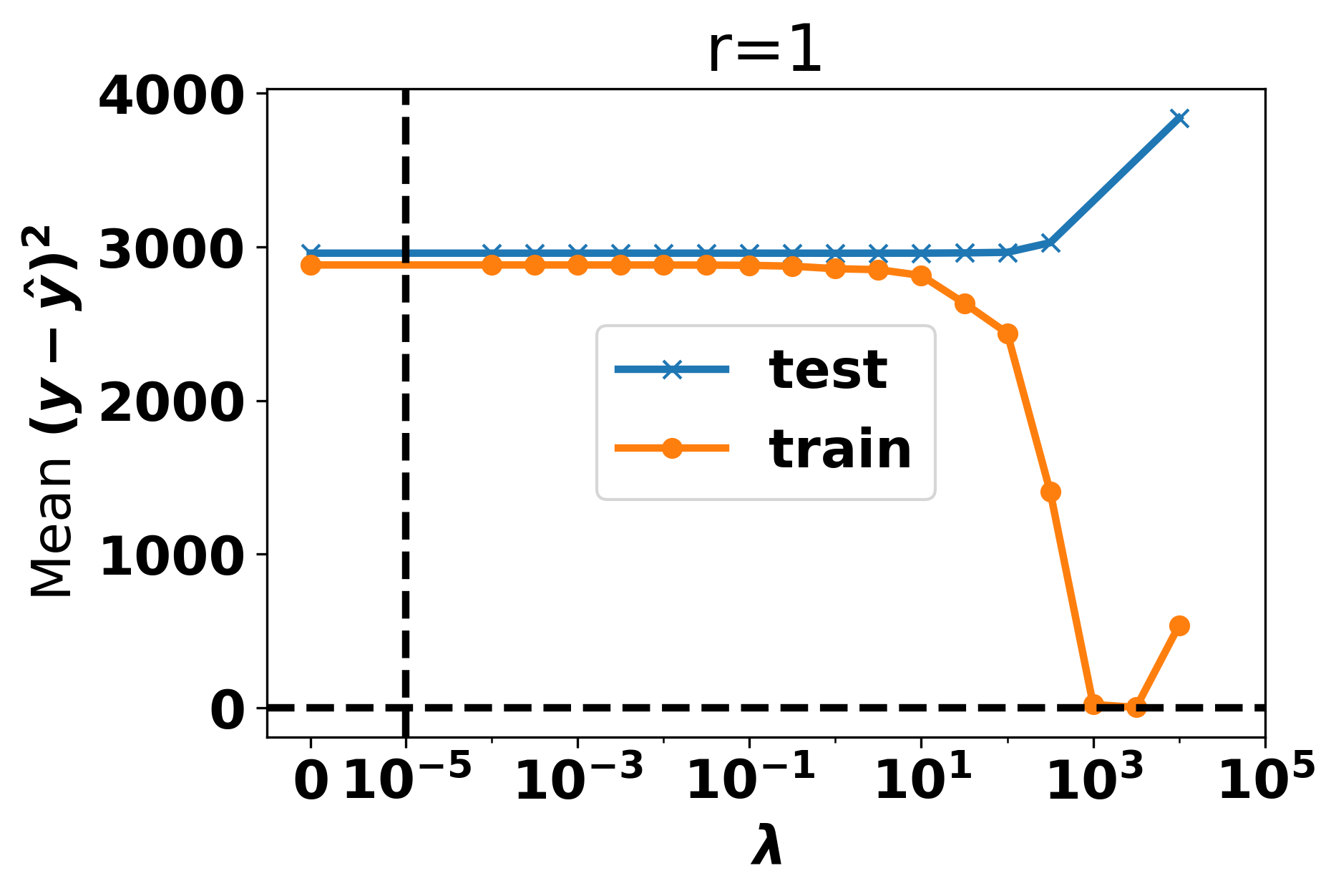}
         \caption{}
         \label{fig:s6}
     \end{subfigure} \hfill
     \begin{subfigure}[b]{0.3\textwidth}
         \centering
         \includegraphics[width=\textwidth]{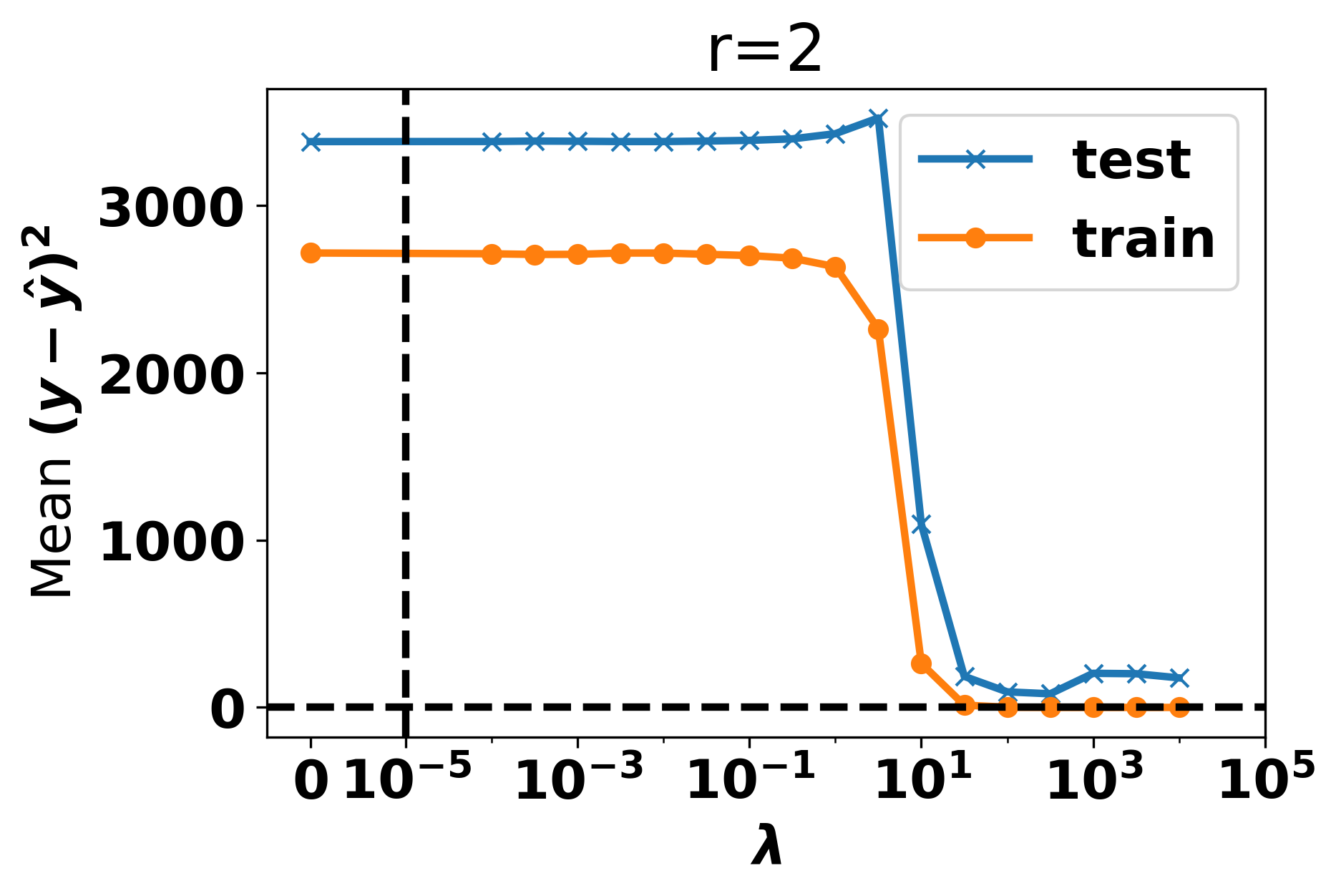}
         \caption{}
         \label{fig:s5}
     \end{subfigure} \\
     \begin{subfigure}[b]{0.3\textwidth}
         \centering
         \includegraphics[width=\textwidth]{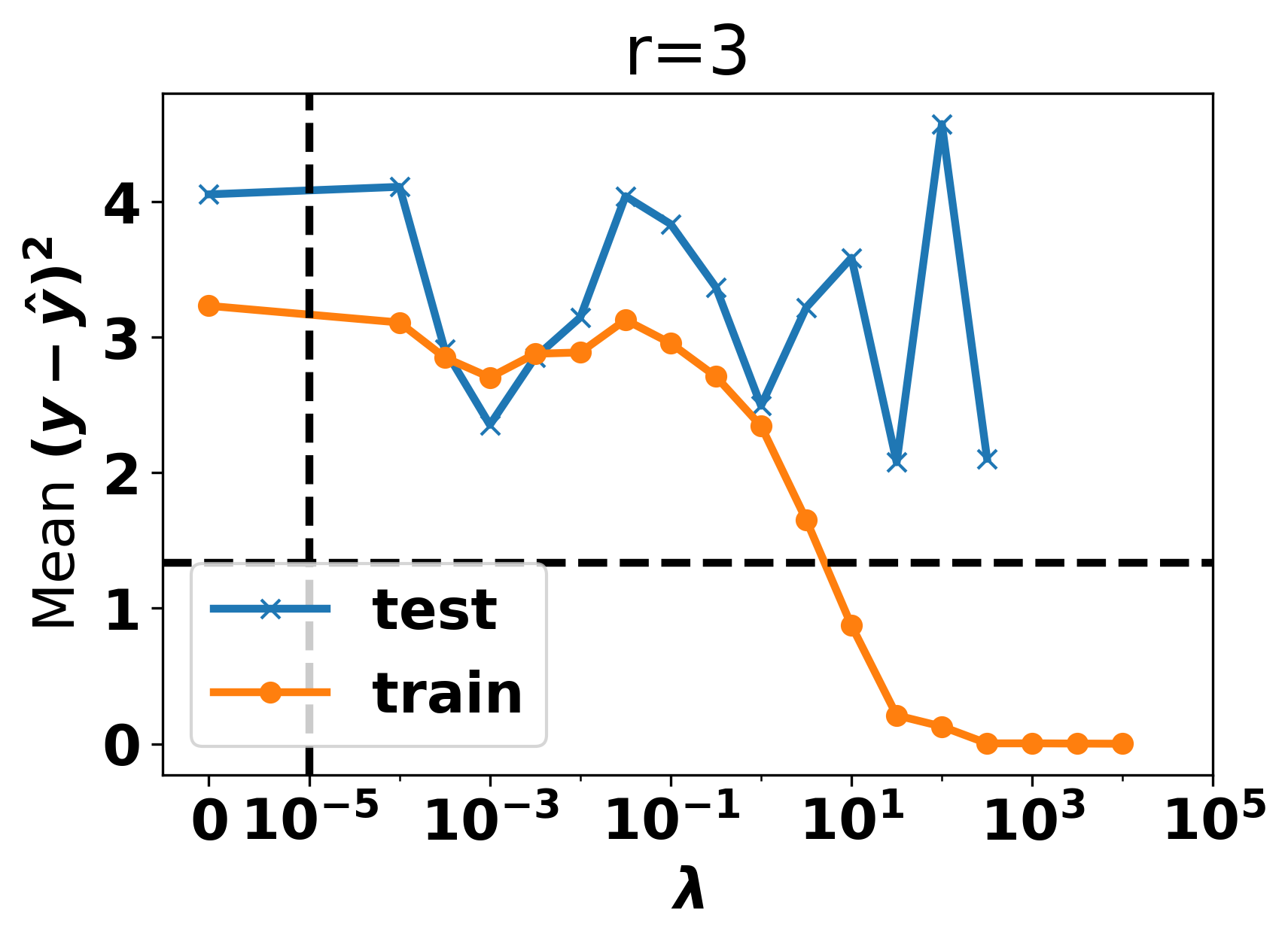}
         \caption{}
         \label{fig:s6}
     \end{subfigure} \hfill
          \begin{subfigure}[b]{0.3\textwidth}
         \centering
         \includegraphics[width=\textwidth]{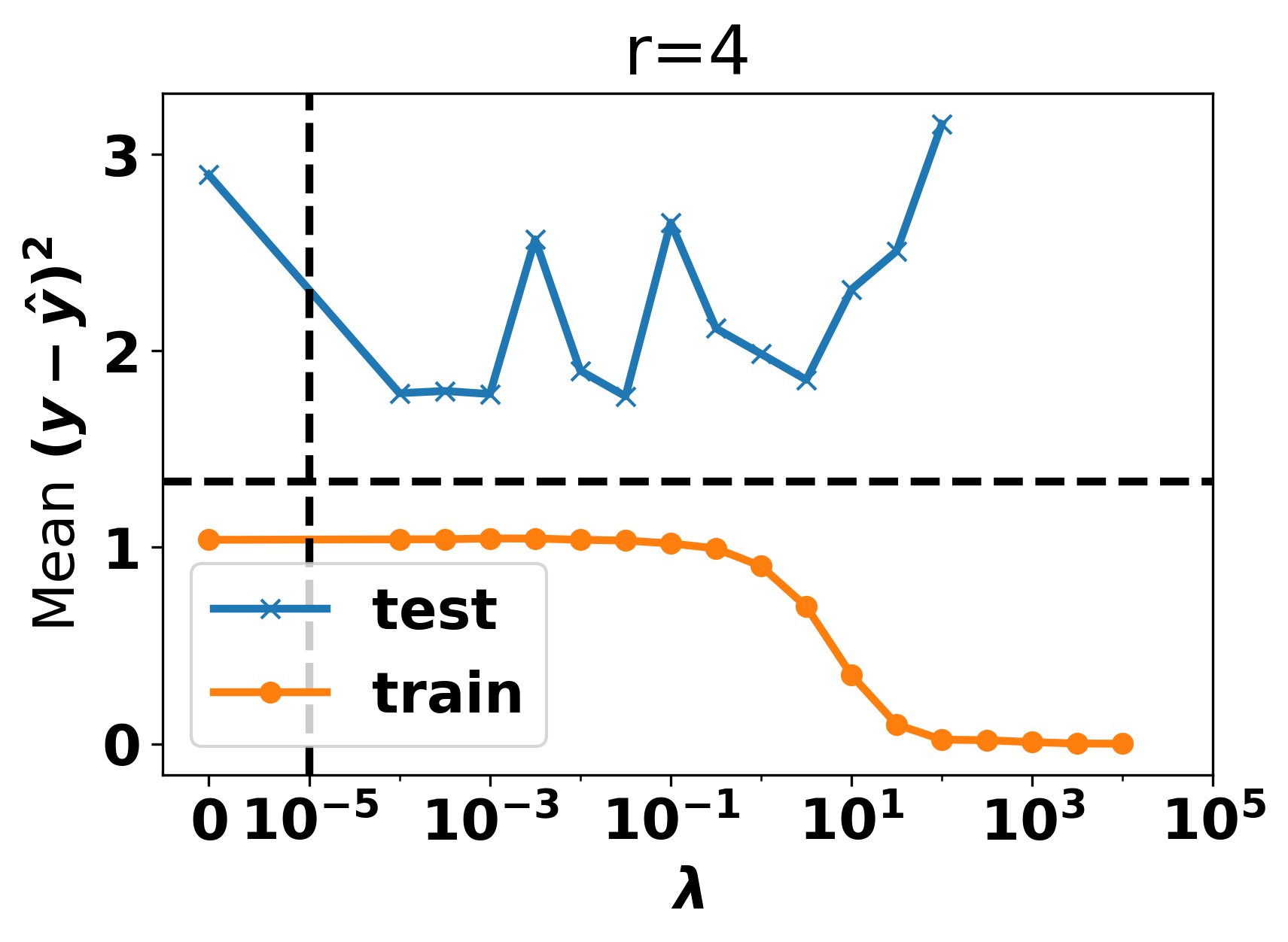}
         \caption{}
         \label{fig:s5}
     \end{subfigure} \hfill
     \begin{subfigure}[b]{0.3\textwidth}
         \centering
         \includegraphics[width=\textwidth]{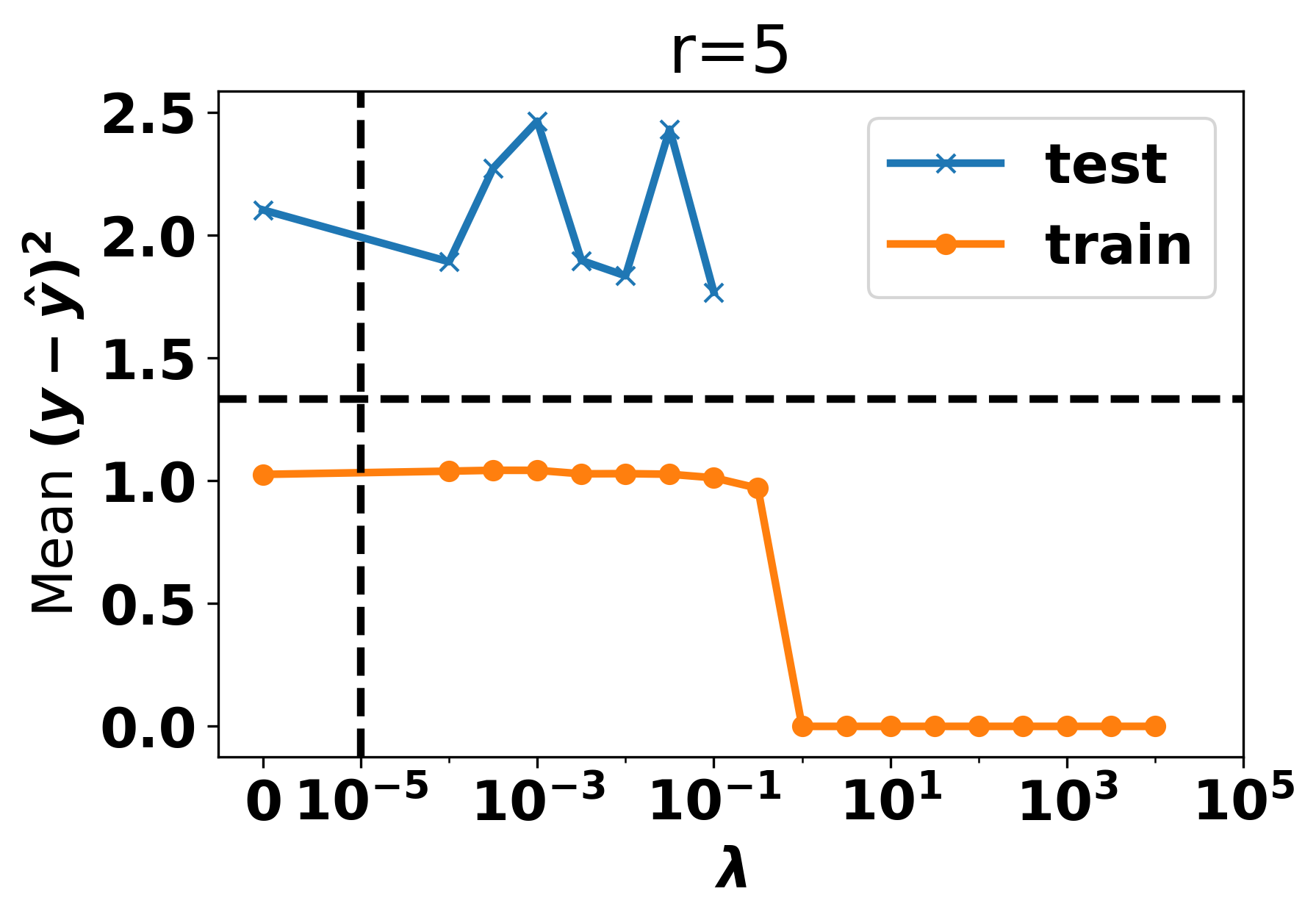}
         \caption{}
         \label{fig:s6}
     \end{subfigure} \\
          \begin{subfigure}[b]{0.3\textwidth}
         \centering
         \includegraphics[width=\textwidth]{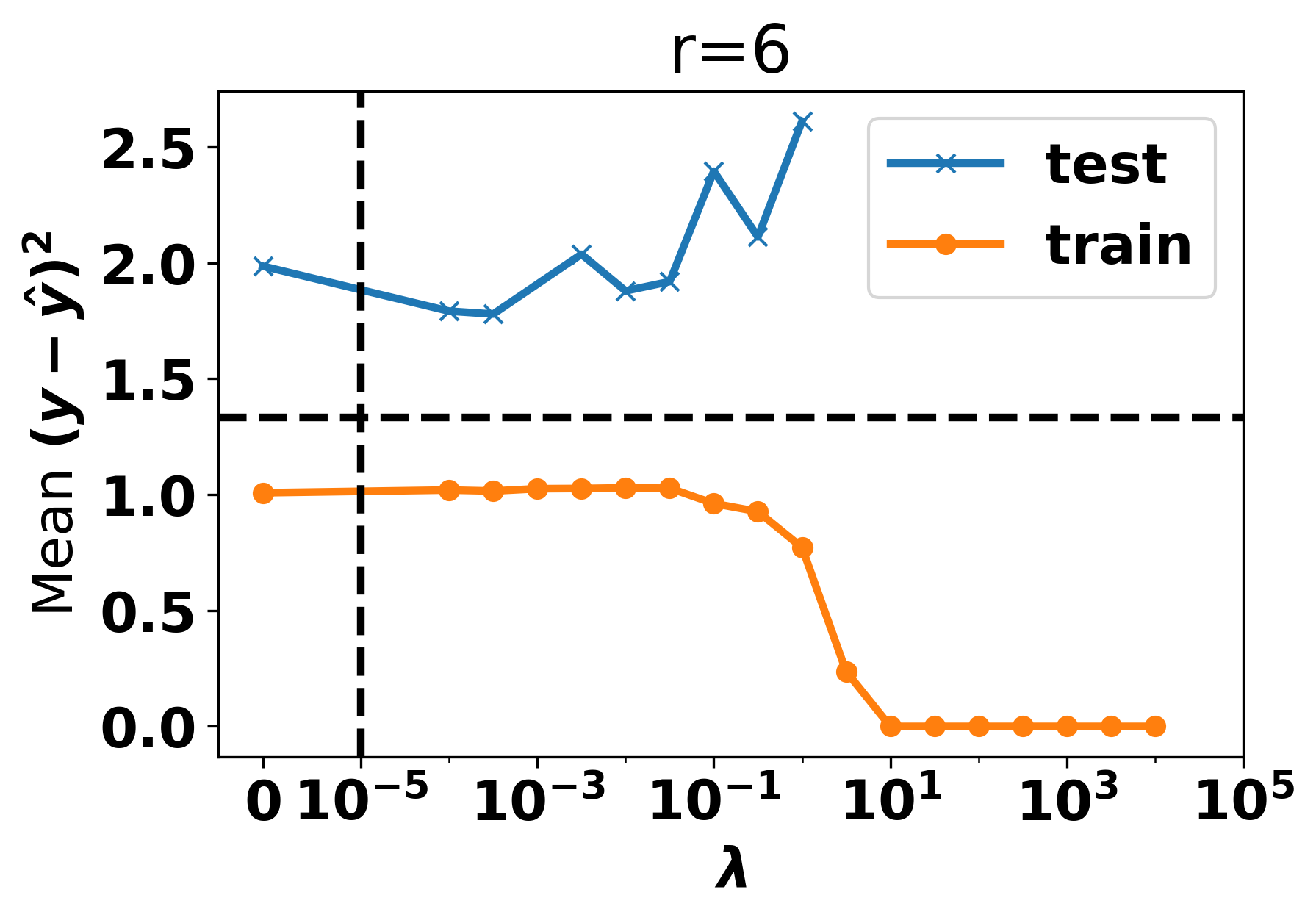}
         \caption{}
         \label{fig:s5}
     \end{subfigure}
     \hfill
     \begin{subfigure}[b]{0.3\textwidth}
         \centering
         \includegraphics[width=\textwidth]{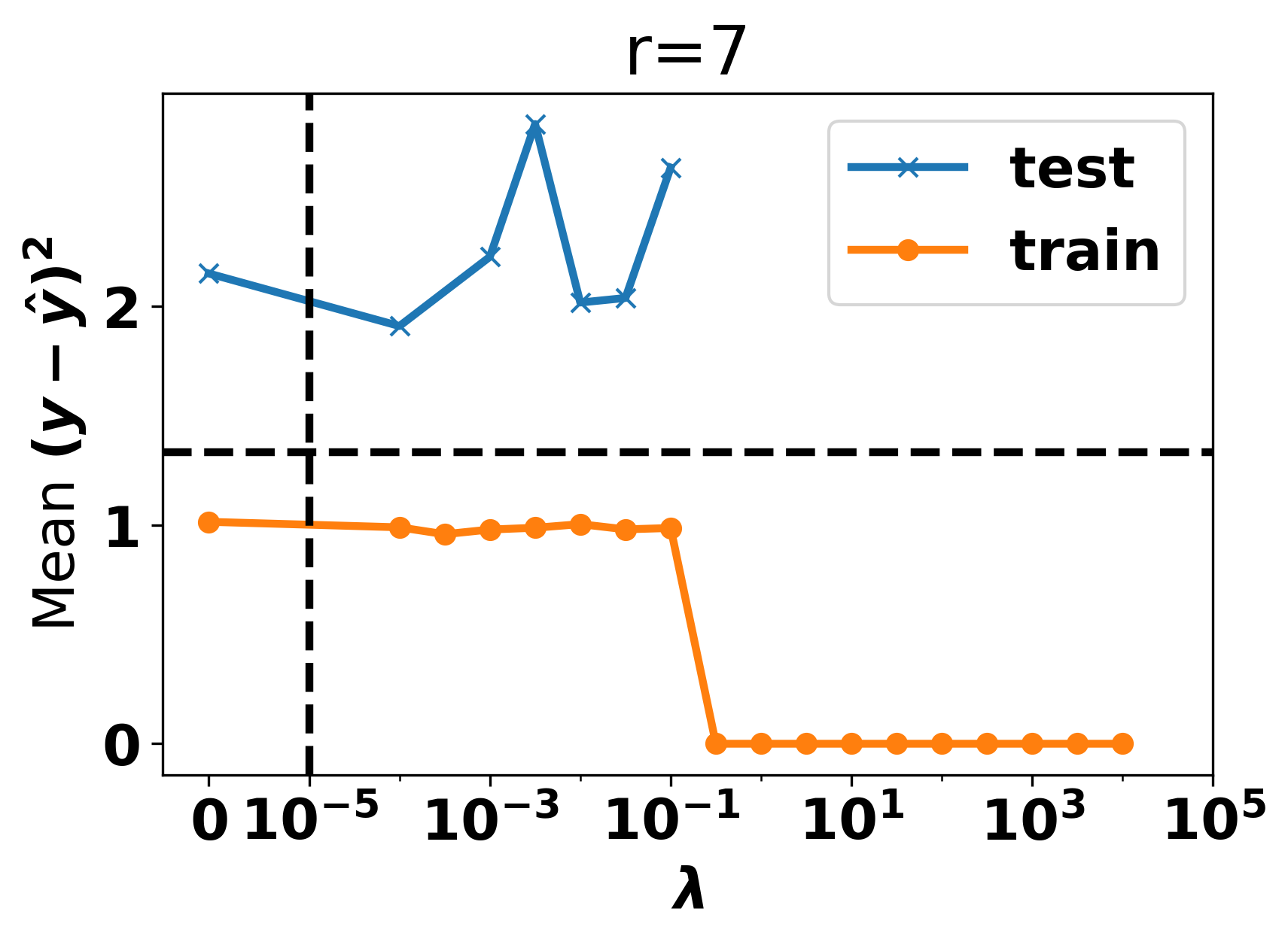}
         \caption{}
         \label{fig:s6}
     \end{subfigure} \hfill \hfill 
     \begin{subfigure}[b]{0.3\textwidth}
     \end{subfigure}      
        \caption{Regression errors with varying regression weight $\lambda$ with different numbers of topics $r$ for Gaussian noise (a)-(g) and Uniform noise (h)-(n). The $\lambda$ values used are the set $\{0\} \cup \{ 10^{i/2} | i \in \mathbb{Z} \cap [-8,8] \}.$ For each $\lambda$ and $r$, fifty trials were run and the regression errors corresponding to the best overall objective function $F^{(\lambda)}$ were recorded. Points with $\lambda>0$ for which the regression error exceeds $1.5$ times the regression error at $\lambda=0$ are not displayed. The dashed horizontal line is the estimated minimal mean regression error. The dashed vertical line is the transition point between a linear and logarithmic $x-$scale.}
        \label{fig:synthetic_errors}
\end{figure}

Taken together, we anticipate that CSSNMF will perform well provided the number of topics chosen does not exceed the true number of topics in the dataset (difficult to assess). We expect that the optimal predictions on unseen data should occur at a $\lambda$ large enough that the testing errors have decreased and plateaued. In Figure \ref{fig:synthetic_errors}, we see that for large $\lambda$, when overfitting is an issue, the testing performance is seldom better than where $\lambda=0$ (classical NMF and then regression) and in fact is often much worse.

\section{Rate My Professors Dataset}

\label{sec:rmp}

\subsection{Pre-processing}

The corpus was first processed via TFIDF \cite{joachims1996probabilistic} with the TfidfVectorizer class in \texttt{Python}'s \texttt{scikit learn} package \cite{scikit-learn}. We used arguments {\bf min\_df=0.01}, {\bf max\_df=0.15}, {\bf stop\_words='english'}, {\bf norm=`l1'}, {\bf lowercase=True}. We found the ratings were not balanced: there were $57$ on the interval $[1,2)$, $235$ on the interval $[2,3)$, $494$ on the interval $[3,4)$, and $629$ on the interval $[4,5]$. To balance the dataset, we extracted only a random subset of $57$ reviews in each interval (all ratings on $[1,2)$ were used). Overall, we obtained a corpus matrix $X$ that was $228 \times 1635$. The open right-end of the intervals ensures data are not duplicated.

\subsection{Choice of Topic Number and Regression Weight}

We did not know the true number of topics in the dataset and chose topics of $r=1,3,5,7,9$, and $11$ with $\lambda \in \{0\} \cup \{10^{2i/3} | i \in [-12,0] \cap \mathbb{Z} \}$. We present the results for $11$ topics which gave the best results. See Figure \ref{fig:rmp_errs}. We note that for large enough $\lambda$, the testing error outperforms the testing error for $\lambda=0$. The optimal point was at $\lambda = 10^{-2/3} \approx 0.215$.

We comment that it is generally difficult to know precisely where the testing error will be minimized, only that, based on observations of the synthetic data, the testing error is often better than the $\lambda=0$ case after the training error has dropped. We speculate that the level of noise in this dataset results in the testing errors not dropping below $\approx 0.75$.

\begin{figure}
\centering
\includegraphics[width=0.7\linewidth]{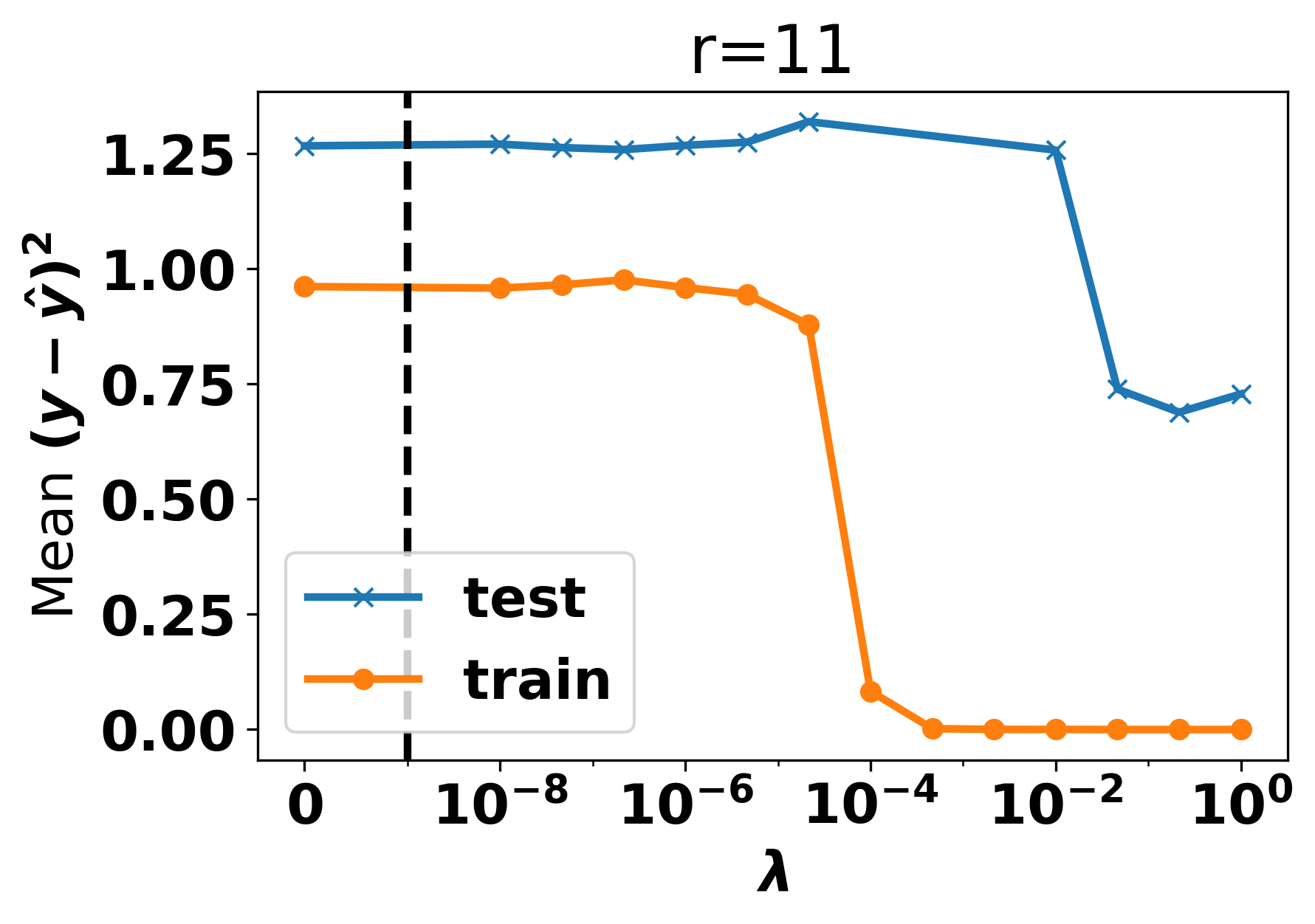}.
\caption{Errors in training and validation on Rate My Professor dataset with $r=11$ topics. Points with $\lambda>0$ for which the regression error exceeds $1.5$ times the regression error at $\lambda=0$ are not displayed. The dashed vertical line is the transition point between a linear and logarithmic $x-$scale.}
\label{fig:rmp_errs}
\end{figure}

\subsection{Prediction}

We examine the rating prediction by plotting histograms of predicted ratings where the true ratings were in $[1,2], [2,3], [3,4],$ and $[4,5]$ --- the closed intervals are used here. Figure \ref{fig:rmp_hists} depicts these histograms along with the mean predicted rating and true rating. The predictions are often within range and the mean predicted values are very close to the true means over each interval. We can also see the general predictive strength in the scatterplot of actual vs predicted ratings in Figure \ref{fig:rmp_scatter}.

These results suggest the model is able to identify topics and associated $\theta-$ weights so as to generate predictions that are consistent with true ratings. For example, in the case ratings are in $[1,2]$, we see the peak of the predictions is around $2$, not exceeding $4$, with some predictions as low as $-2$; then, in the case of ratings in $[4,5]$, the model peaks around $3.5$ and makes some predictions above $7$. There is a clear capacity for the topics to shift the predictions.

\begin{figure}
     \centering
     \begin{subfigure}[b]{0.45\textwidth}
         \centering
         \includegraphics[width=\textwidth]{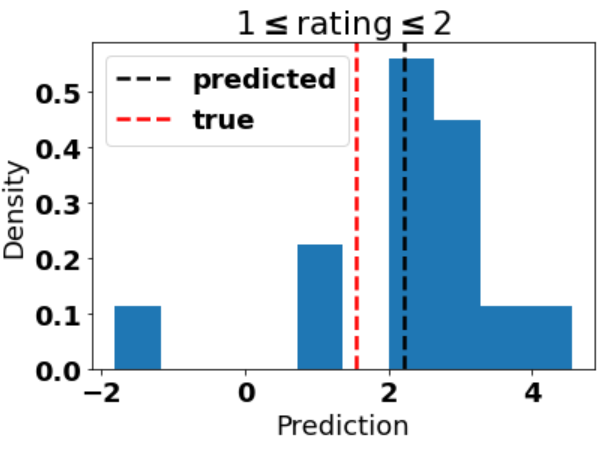}
         \caption{}
         \label{fig:s1}
     \end{subfigure}
     \hfill
     \begin{subfigure}[b]{0.45\textwidth}
         \centering
         \includegraphics[width=\textwidth]{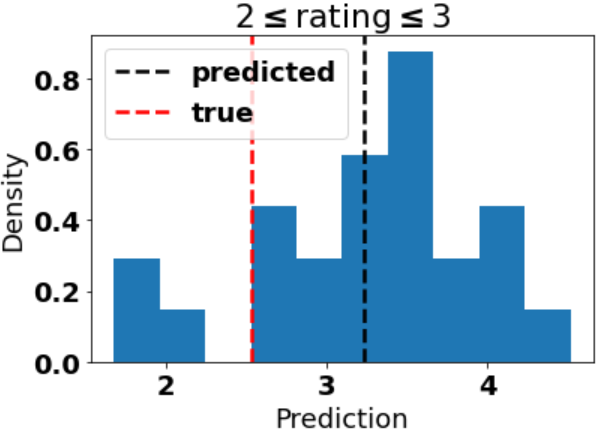}
         \caption{}
         \label{fig:s2}
     \end{subfigure} \\
     \begin{subfigure}[b]{0.45\textwidth}
         \centering
         \includegraphics[width=\textwidth]{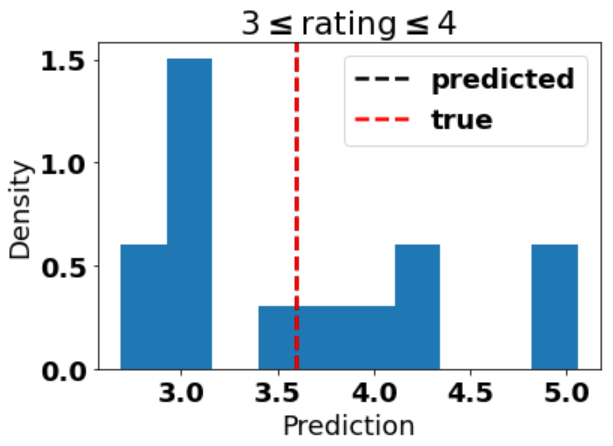}
         \caption{}
         \label{fig:s3}
     \end{subfigure} \hfill
     \begin{subfigure}[b]{0.45\textwidth}
         \centering
         \includegraphics[width=\textwidth]{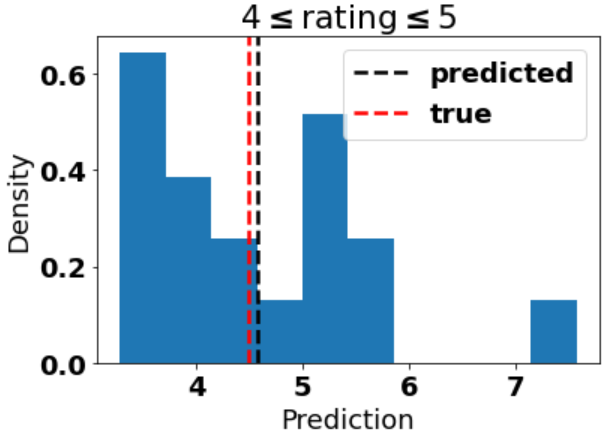}
         \caption{}
         \label{fig:s4}
     \end{subfigure}
        \caption{Histograms of the predicted rating for various ranges of true ratings. The vertical dashed lines represent the mean values. The predicted and true means are as follows: $2.206$ and $1.543$ for ratings in $[1,2]$, $3.233$ and $2.529$ for ratings in $[2,3]$, $3.594$ and $3.593$ for ratings in $[3,4]$ (the lines are indistinguishable), and $4.576$ and $4.494$ for ratings in $[4,5]$.}
        \label{fig:rmp_hists}
\end{figure}

\begin{figure}
\centering
\includegraphics[width=0.5\linewidth]{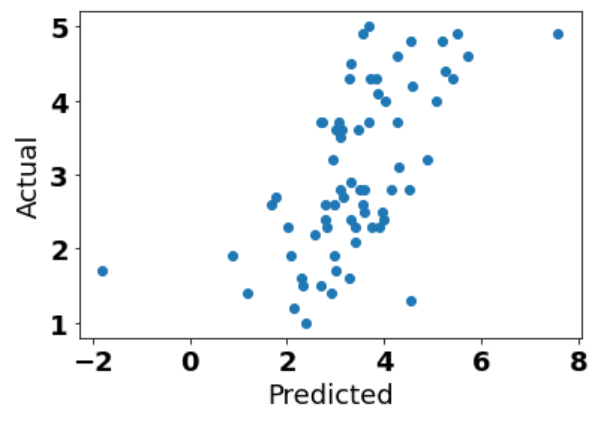}
\caption{Scatterplot of Rate My Professor ratings.}
\label{fig:rmp_scatter}
\end{figure}

\subsection{Topics Identified}

It is important that the method not only have predictive power, but also produce interpretable topics. We now look at the $11$ topics, with their associated $\theta-$weights. We find \begin{multline*}\theta = (2.39909812, 2.82948873, -2.21028471, 1.83876976, -4.77504984, \\
-3.86467795, 3.46353642,  0.03914383,  3.26619842, -5.51595505, \\
4.15317532,  3.90733652)^T.\end{multline*} Note that $\theta_1 \approx 2.4$ suggests that for a set of reviews with no topics, the average rating would be $2.4$ --- this suggests it is the presence of positive/negative topics that raise/lower the rating.

In Figures \ref{fig:rmp_positive} and \ref{fig:rmp_negative}, we plot the words in the topics associated with positive and negative ratings. The topics are interpretable. For positive topics, we find Topic 10 (extra credit) and Topic 11 (being nice/enjoyable class) and words like ``recommend" in a couple of them. A few words seem out of place like ``hate" in Topic 11, but that can be explained by some positive reviews having phrases like ``i hated chemistry in high school and after taking her class i don t {\it [sic]} hate chem as much." Among the negative topics we see Topic 4 (being horrible) and Topic 5 (being unfair).

As a whole, the topics are consistent with intuitive notions of what would be associated with higher or lower ratings. It is also interesting to look at the $\theta$-topic weights quantitatively. For example, both rants and sarcasm (suggested by Topic 2) and being harder and failing students (suggested by Topic 4) contribute negatively to the score, but being a harder teacher seems to contribute more negatively to the rating than ranting.

\begin{figure}
\centering
\includegraphics[width=1.\linewidth]{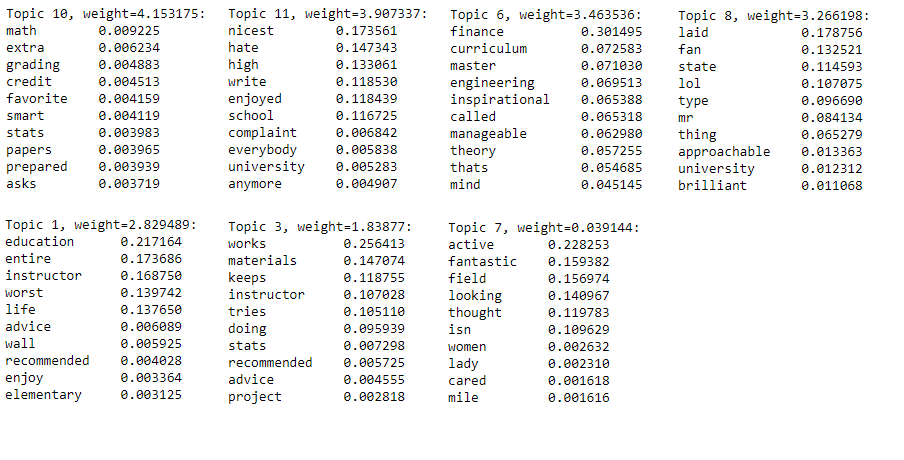}
\caption{Topics with positive $\theta-$weights. The $\theta-$ weight is given as the topic weight. The strength of each word is given numerically beside each of the top 10 words.}
\label{fig:rmp_positive}
\end{figure}

\begin{figure}
\centering
\includegraphics[width=1.\linewidth]{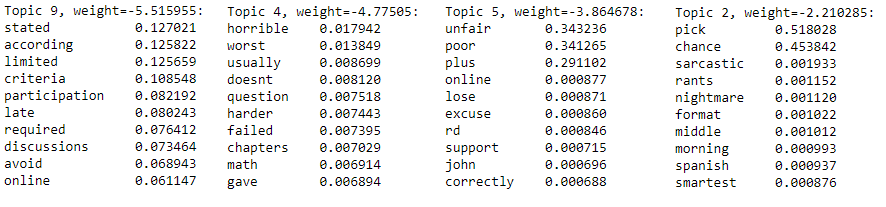}
\caption{Topics with negative $\theta-$weights. The $\theta-$ weight is given as the topic weight. The strength of each word is given numerically beside each of the top 10 words.}
\label{fig:rmp_negative}
\end{figure}

\section{Conclusion and Future Work}

\label{sec:conclusion}

We have developed CSSNMF as a means to combine NMF with regression on a continuous response variable. We accomplished this by minimizing an objective function that combines an NMF error with a weighted regression error. We have shown that the regression error is weakly decreasing with the regression error weight and that, in practical applications, the error in fact strictly decreases. The topics identified can outperform the quantitative accuracy of topics formed through NMF alone while retaining a high degree of interpretability.

While our analysis focused on the case of linear regression, incorporating nonlinearities would be of interest. We also noted the challenge in choosing the appropriate $\lambda$ given only training data. A more theoretical understanding of when testing errors drop substantially could be explored but this may be dataset-specific.


\end{document}